\documentclass[twoside,11pt]{article}

\usepackage[preprint]{dmlr2e}            

\usepackage{amsmath,amssymb,amsthm}
\usepackage{booktabs}
\usepackage{graphicx}
\usepackage{hyperref}
\usepackage{xcolor}
\usepackage{multirow}
\usepackage{array}
\usepackage{enumitem}
\usepackage[protrusion=false,expansion=false]{microtype}
\usepackage{caption}
\usepackage{subcaption}

\hypersetup{colorlinks=true, linkcolor=blue, citecolor=blue, urlcolor=blue}

\graphicspath{{figures/}}

\newtheorem{proposition}{Proposition}
\newtheorem{remark}{Remark}

\title{Synthetic Tabular Generators Fail to Preserve Behavioral Fraud Patterns:
A Benchmark on Temporal, Velocity, and Multi-Account Signals}

\author{%
  \name Bhavana Sajja \email sajja.bhavana@gmail.com \\
  \addr Independent Researcher
}

\ShortHeadings{Synthetic Tabular Generators Fail to Preserve Behavioral Fraud Patterns}{Sajja}

\date{}

\begin{document}

\maketitle

\begin{abstract}
We introduce \emph{behavioral fidelity}---a third evaluation dimension for synthetic tabular
data that measures whether generated data preserves the temporal, sequential, and structural
behavioral patterns that distinguish real-world entity activity.
Existing frameworks evaluate \emph{statistical fidelity} (marginal distributions and
correlations) and \emph{downstream utility} (classifier AUROC on synthetic-trained models),
but neither tests for the behavioral signals that operational detection and analysis systems
actually rely on.

We formalize a taxonomy of four behavioral fraud patterns (P1--P4) covering inter-event
timing, burst structure, multi-account shared-infrastructure graph motifs, and velocity-rule
trigger rates; define a \emph{degradation ratio} metric calibrated to a real-data noise
floor (1.0 = matches real variability, $k$ = $k$-times worse); and prove that
row-independent generators---the dominant paradigm---are \emph{structurally incapable}
of reproducing P3 graph motifs (Proposition~1) and produce non-positive within-entity IET
autocorrelation (Proposition~2), making the positive burst fingerprint of fraud sequences
unachievable regardless of architecture, training data size, or post-processing.

We benchmark four generators---CTGAN, TVAE, GaussianCopula \cite{sdv}, and TabularARGN
\cite{tabularargn}---on IEEE-CIS Fraud Detection \cite{ieee_cis} and the Amazon Fraud
Dataset \cite{amazon_fdb}.
All four generators fail severely: on IEEE-CIS (P1, P2, P4) composite degradation ratios
range from 24.4$\times$ (TVAE, after conditional sampling correction) to
39.0$\times$ (GaussianCopula); on Amazon FDB (P3), row-independent generators score
81.6--99.7$\times$, while TabularARGN's autoregressive architecture achieves 17.2$\times$
with full-column training.
We additionally document generator-specific failure modes (TVAE minority-class collapse,
CTGAN high-dimensional scalability failure) and their practical resolutions.
These findings hold implications beyond fraud: the P1--P4 framework extends directly to
any domain with entity-level sequential tabular data, including healthcare records,
e-commerce behavior, and network security.
We release our evaluation framework as open source to enable reproducible behavioral
fidelity assessment across domains.
\end{abstract}

\keywords{synthetic tabular data, fraud detection, behavioral fidelity, generative models,
tabular data generation, CTGAN, TVAE, evaluation framework, graph motifs, velocity rules}

\section{Introduction}
\label{sec:intro}

Financial fraud detection is a behavioral problem.
Production fraud systems flag accounts based on how their transaction \emph{sequences}
deviate from baseline: a card executing three transactions in 60 seconds; a cluster
of user accounts registered within hours and sharing the same IP address; an amount
spike ratio of 10$\times$ a cardholder's 30-day median.
These behavioral signals---temporal bursts, velocity rule violations, shared infrastructure
---are the operational basis of fraud detection \cite{bahnsen, xfraud, dalpozzolo}.

When raw transaction data cannot be shared due to privacy regulations such as GDPR
\cite{gdpr}, synthetic generation is the natural substitute.
The central assumption enabling this substitution is that generators
\emph{preserve the structure that matters for fraud detection}.
We demonstrate that this assumption is largely untested and, where we can measure it,
substantially violated.

\paragraph{The evaluation gap.}
Existing synthetic data benchmarks \cite{tabddpm, syntheval, stoian2026}
measure two properties: (1)~\emph{statistical fidelity}---whether marginal distributions
and pairwise correlations match the real data; and (2)~\emph{downstream utility}---whether
a classifier trained on synthetic data generalizes to real data (the TSTR protocol
\cite{tstr}).
Both are necessary but insufficient for fraud applications.
A generator can exactly reproduce the marginal distribution of transaction amounts and
pairwise amount-timestamp correlations while completely destroying the within-entity burst
structure that distinguishes card testers from ordinary cardholders.
AUROC on a held-out test set averages over all transactions equally; it does not
specifically test velocity-rule calibration, graph-based ring detection, or
sequence-level anomaly scoring.
We demonstrate this gap concretely: CTGAN achieves the second-highest TSTR AUROC
(0.798, near the real-data baseline of 0.903) yet the worst P3 degradation ratio
(99.7$\times$) in the benchmark.
GaussianCopula achieves the lowest TSTR AUROC (0.523) yet a better P3 score (81.6$\times$).
No Layer~1 or Layer~2 score predicts Layer~3 performance in any consistent direction.

This gap has direct operational consequences.
Consider a velocity rule that flags cards with $>3$ transactions within one hour.
In CTGAN synthetic data, rule R1 fires at an absolute rate that is 0.36 points lower
than in real fraud data.
A threshold tuned to minimize false positives on this synthetic data will be far too
permissive when deployed against real fraud, which triggers such rules at materially
higher rates.
The same disconnect applies to graph-based detectors trained on synthetic data with
randomly assigned device IDs and IP addresses: the learned ring signatures have no
relationship to actual fraud ring structure.

\paragraph{Prior work on synthetic fraud data.}
PaySim \cite{paysim} simulates mobile money transactions using agent-based rules that
explicitly encode behavioral fraud patterns, demonstrating that rule-based simulation
can preserve behavioral structure---but it is hand-coded simulation, not learned
generation, and cannot generalize to new fraud patterns.
\citet{new_money} benchmark tabular generators on financial transaction data and identify
evaluation gaps around temporal and behavioral properties, which our framework directly
addresses with formal definitions.
\citet{cpar_fraud} evaluate an autoregressive generator on credit card transaction
sequences using statistical fidelity metrics, but do not define a behavioral fraud
taxonomy, do not measure velocity rules or graph motifs, and study only a single generator.
\citet{fraudgan} apply GAN-based oversampling to credit card fraud data and evaluate
classifier performance, not temporal structure.
To our knowledge, no prior work benchmarks multiple tabular generators against a formal
behavioral fraud pattern taxonomy spanning temporal, velocity, \emph{and} graph-structural
dimensions.

\paragraph{Contributions.}
We make four contributions:

\begin{enumerate}[leftmargin=*,noitemsep]
  \item \textbf{Behavioral fraud pattern taxonomy} (Section~\ref{sec:taxonomy}): Four
  formally defined, measurable patterns grounded in the fraud detection literature
  \cite{bahnsen, xfraud, dalpozzolo, fraud_survey}---inter-event time distribution (P1),
  burst structure and active lifetime (P2), shared-infrastructure graph motifs (P3), and
  velocity-rule trigger rates (P4)---with precise metric definitions enabling
  cross-generator and cross-dataset comparison.

  \item \textbf{Degradation-ratio evaluation framework} (Section~\ref{sec:framework}):
  A noise-floor-anchored scoring system that normalizes raw behavioral metrics into
  interpretable ratios, together with a three-layer evaluation protocol (statistical
  fidelity + downstream utility + behavioral fidelity) that exposes the gap between
  existing dimensions and behavioral reality.

  \item \textbf{Empirical benchmark} (Section~\ref{sec:results}): Evaluation of CTGAN,
  TVAE, GaussianCopula, and TabularARGN on IEEE-CIS (P1, P2, P4) and Amazon FDB (P3).
  On IEEE-CIS, composite degradation ratios range from 24.4$\times$ (TVAE) to
  39.0$\times$ (GaussianCopula); all four generators fail severely relative to the
  real-data noise floor.
  On Amazon FDB (P3), row-independent generators range from 81.6--99.7$\times$;
  TabularARGN achieves 17.2$\times$ with full-column training, demonstrating that
  autoregressive architecture provides a measurable but insufficient improvement for
  graph motif preservation.

  \item \textbf{Failure mode documentation} (Section~\ref{sec:discussion}): TVAE
  minority-class collapse under unconditional sampling (and its resolution via conditional
  sampling), CTGAN high-dimensional scalability failure, the architectural advantage of
  TabularARGN for P3 graph motifs versus its inability to improve P1/P2/P4 temporal
  patterns, and the theoretically-grounded impossibility of cross-entity pattern
  preservation for row-independent generators---with actionable practitioner guidance.
\end{enumerate}

\section{Related Work}
\label{sec:related}

\paragraph{Tabular synthetic data generation.}
CTGAN and TVAE \cite{ctgan} established the conditional GAN and VAE paradigm for tabular
data, using a conditional vector and mode-specific normalization to handle multi-modal
continuous distributions and class imbalance.
\citet{sdv_original} introduced the Synthetic Data Vault, of which GaussianCopula is a
component: it models each column's marginal independently and captures pairwise dependencies
via a parametric Gaussian copula.
Diffusion-based methods---TabDDPM \cite{tabddpm} and TabSyn \cite{tabsyn}---achieve
state-of-the-art statistical fidelity on tabular benchmarks but have not been evaluated
for behavioral fraud pattern preservation.
CTAB-GAN \cite{ctabgan} improves on CTGAN's handling of mixed-type columns and long-tail
distributions; REaLTabFormer \cite{realtabformer} uses a transformer to capture
relational structure between tables.
TabularARGN \cite{tabularargn} conditions each feature on previously generated features
within a row via an autoregressive architecture.
None of these generators are evaluated against behavioral fraud metrics in their original
publications.

\paragraph{Synthetic data evaluation.}
SDMetrics \cite{sdv} measures column-shape, pairwise-trend, and ML-efficacy scores.
SynthEval \cite{syntheval} unifies fidelity, utility, and privacy evaluation.
\citet{tabddpm} benchmark multiple generators across 15 datasets in their evaluation;
the TSTR protocol \cite{tstr} is now standard practice in the field.
\citet{stoian2026} provide a comprehensive taxonomy of deep learning approaches for
tabular data generation, organizing evaluation dimensions as utility, alignment, fidelity,
privacy, and diversity---the most thorough recent survey of the field.
Their framework introduces formal numbered requirements for each dimension and evaluates
a wide range of generator architectures, including GAN-based, VAE-based, diffusion-based,
and LLM-based models.
Crucially, however, behavioral fraud patterns---inter-event timing, burst structure,
shared-infrastructure graph motifs, and velocity-rule calibration---are not among the
evaluated dimensions in any of the above frameworks.
This is the gap our work addresses.

\paragraph{Synthetic financial transaction data.}
PaySim \cite{paysim} uses agent-based simulation to produce mobile money transactions
that explicitly encode fraud behavior (burst patterns, rapid withdrawals), but produces
pre-programmed rather than learned distributions.
\citet{new_money} is the closest prior work to ours: they benchmark tabular generators
on financial transaction data and identify evaluation gaps around temporal properties.
However, they do not formalize a behavioral fraud taxonomy, do not evaluate velocity rules
or graph motifs, and do not report degradation ratios relative to a noise floor.
\citet{cpar_fraud} evaluate an autoregressive generator on credit card sequences using
statistical metrics only, with no velocity rule or graph evaluation.

\paragraph{Fraud detection and behavioral signals.}
The fraud detection literature provides strong evidence that behavioral features are the
most discriminative signals \cite{fraud_survey}.
\citet{bahnsen} show that aggregate time-based features---transaction counts per rolling
window, time elapsed since last event---are among the strongest predictors of credit card
fraud, outperforming static features.
\citet{dalpozzolo} demonstrate from a practitioner perspective that recency-weighted
velocity aggregates are essential to fraud detection system performance.
Graph-based detection \cite{xfraud} exploits shared device and IP address structure to
identify coordinated fraud rings, constructing explicit bipartite entity-attribute graphs
and detecting high-density subgraphs as fraud ring signatures.
These bodies of work define precisely the behavioral signals whose preservation we test.

\section{Behavioral Fraud Pattern Taxonomy}
\label{sec:taxonomy}

\subsection{Notation}

Let $\mathcal{U}$ be the set of transaction entities (card fingerprints, user accounts).
For entity $u \in \mathcal{U}$, let
$\mathcal{T}_u = \langle (t_1^u, \mathbf{x}_1^u), \ldots, (t_{n_u}^u, \mathbf{x}_{n_u}^u) \rangle$
be its chronologically ordered transaction sequence, with $t_1^u < \cdots < t_{n_u}^u$.
Let $\mathcal{F} \subset \mathcal{U}$ be fraud entities and
$\mathcal{N} = \mathcal{U} \setminus \mathcal{F}$ non-fraud entities.
For a class label $c \in \{F, N\}$, let $\mathcal{C}$ denote the corresponding entity set
($\mathcal{F}$ or $\mathcal{N}$).
$D_{\text{real}}$ and $D_{\text{syn}}^G$ denote real and generator-$G$-produced synthetic
datasets; tilde notation ($\tilde{\mathcal{F}}$, $\tilde{\mathcal{N}}$) denotes synthetic
counterparts.
$W_1(P, Q)$ denotes the Wasserstein-1 (Earth Mover's) distance \cite{wasserstein} between
empirical distributions $P$ and $Q$.

\subsection{P1: Inter-Event Time Distribution}
\label{sec:p1}

\paragraph{Motivation.}
Account-takeover, card-testing, and credential-stuffing attacks share a temporal fingerprint:
rapid bursts of transactions with compressed inter-event gaps, followed by silence
\cite{bahnsen}.
A generator that reproduces the marginal timestamp distribution but destroys within-entity
temporal ordering fails to capture this distinguishing structure.

\paragraph{Definition.}
For entity $u$ with $n_u \geq 2$ transactions, the inter-event time (IET) sequence is:
\begin{equation}
  \Delta_u = \langle \delta_i^u \rangle_{i=1}^{n_u-1}, \qquad
  \delta_i^u = t_{i+1}^u - t_i^u \;\geq\; 0.
\end{equation}
The class-level IET distribution is
$\mathrm{IETD}^c(D) = \bigcup_{u \in \mathcal{C}} \Delta_u$.
Within-entity temporal autocorrelation at lag~1 captures burst regularity:
\begin{equation}
  \rho_u = \operatorname{Corr}\!\left(\langle\delta_i^u\rangle_{i=1}^{n_u-2},\;
                                      \langle\delta_{i+1}^u\rangle_{i=1}^{n_u-2}\right).
\end{equation}

\paragraph{Metrics.}
\begin{align}
  B_1^c(G) &= W_1\!\left(\mathrm{IETD}^c(D_{\text{real}}),\;
                          \mathrm{IETD}^c(D_{\text{syn}}^G)\right), \\
  B_1^{\mathrm{AC}}(G) &= \left|\mathbb{E}_{u \in \mathcal{F}}[\rho_u]
                                -\mathbb{E}_{u \in \tilde{\mathcal{F}}}[\rho_u]\right|.
\end{align}
$B_1^c$ measures distributional shift in IETs; $B_1^{\mathrm{AC}}$ measures the
collapse of within-entity autocorrelation---the burst-regularity fingerprint of fraud.
We report $B_1^F$ (fraud class) and $B_1^{\mathrm{AC}}$ in the benchmark.

\subsection{P2: Burst Structure and Active Lifetime}
\label{sec:p2}

\paragraph{Motivation.}
Fraudsters operate under time pressure: detection windows close quickly.
This produces short active lifetimes (hours to days) with dense transaction bursts.
Card testers execute a single tight burst of small probes, then disappear.
Legitimate accounts have long active lifetimes with scattered, low-density activity
\cite{dalpozzolo}.

\paragraph{Definition.}
For gap threshold $\delta > 0$ (evaluated at $\delta \in \{1\text{ min}, 5\text{ min},
30\text{ min}\}$), a \emph{burst} is a maximal contiguous subsequence of
$\mathcal{T}_u$ with consecutive gaps $\leq \delta$.
Denote the burst set of entity $u$ at threshold $\delta$ as $\mathcal{B}_u(\delta)$,
with burst length $L(b) = |b|$.
The active lifetime is $AL_u = t_{n_u}^u - t_1^u$ (set to 0 for singleton entities).

\paragraph{Metrics.}
\begin{align}
  B_{2,AL}(G) &= W_1\!\left(
      \{AL_u\}_{u \in \mathcal{F}},\;
      \{\widetilde{AL}_u\}_{u \in \tilde{\mathcal{F}}}\right), \\
  B_{2,BL}(G, \delta) &= W_1\!\left(
      \bigl\{L(b)\bigr\}_{b \in \bigcup_{u \in \mathcal{F}} \mathcal{B}_u(\delta)},\;
      \bigl\{L(b)\bigr\}_{b \in \bigcup_{u \in \tilde{\mathcal{F}}} \mathcal{B}_u(\delta)}\right).
\end{align}
We report $B_{2,AL}$ (fraud active lifetime) and $\bar{B}_{2,BL} = \frac{1}{3}\sum_\delta B_{2,BL}(G,\delta)$
(burst length averaged across all three thresholds).

\subsection{P3: Shared-Infrastructure Graph Motifs}
\label{sec:p3}

\paragraph{Motivation.}
Fraud rings share infrastructure: device IDs, IP addresses, billing addresses.
This creates a characteristic bipartite graph structure---power-law fan-out distributions,
high triangle counts, dense connected components---that is a primary signal for coordinated
fraud detection \cite{xfraud}.
Row-independent generators sample shared attributes from marginal distributions,
which cannot reproduce cross-row co-occurrence structure by construction.

\paragraph{Definition.}
Define a bipartite entity-attribute graph $\mathcal{G} = (U, A, E)$ where $U$ are entity
nodes, $A$ are shared-attribute nodes (device IDs, IP addresses), and $(u, a) \in E$ if
entity $u$ used attribute $a$.
The \emph{fan-out} of attribute $a$ is $\mathrm{FO}(a) = |\{u : (u, a) \in E\}|$.
The entity projection graph $\mathcal{G}_U$ connects two entities if they share at least
one attribute.

\paragraph{Metrics.}
\begin{align}
  B_{3,\mathrm{FO}}(G) &= W_1\!\left(
      \{\mathrm{FO}(a)\}_{a \in A^{\text{real}}},\;
      \{\mathrm{FO}(a)\}_{a \in A^{\text{syn}}}\right), \\
  B_{3,\mathrm{CC}}(G) &= \left|\mathrm{CC}(\mathcal{G}_U^{\text{real}})
                                - \mathrm{CC}(\mathcal{G}_U^{\text{syn}})\right|, \\
  B_{3,\triangle}(G) &= \left|\log\!\left(
      \frac{|\triangle(\mathcal{G}_U^{\text{real}})|+1}
           {|\triangle(\mathcal{G}_U^{\text{syn}})|+1}\right)\right|,
\end{align}
where $\mathrm{CC}(\cdot) \in [0, 1]$ is the global clustering coefficient and
$\triangle(\cdot)$ is the triangle count.
These capture fan-out distributional shift, local clustering structure, and motif density
respectively.
Figure~\ref{fig:bipartite} illustrates the structural difference between a real fraud ring
(multiple users sharing a device node) and the synthetic analog (each user assigned a unique device).

\begin{figure}[t]
  \centering
  \includegraphics[width=0.85\textwidth]{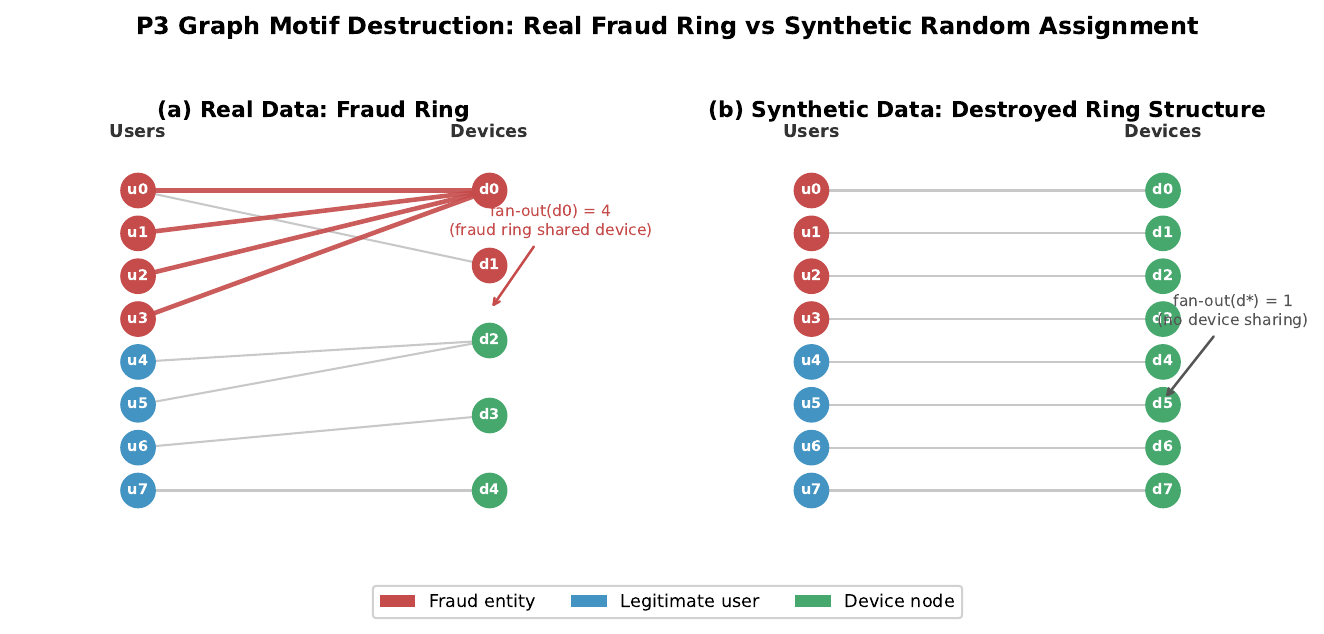}
  \caption{%
    \textbf{P3 graph motif destruction.}
    In real data (left), fraud ring members share device nodes, creating high fan-out and
    dense bipartite structure.
    In synthetic data (right), row-independent generators sample device IDs from a
    marginal distribution, collapsing fan-out to~1 for every device node.
    All tested generators reproduce this failure to varying degrees.
  }
  \label{fig:bipartite}
\end{figure}

\subsection{P4: Velocity-Rule Trigger Rates}
\label{sec:p4}

\paragraph{Motivation.}
Production fraud systems are calibrated around velocity rules \cite{bahnsen, dalpozzolo}.
If synthetic data does not preserve the rates at which fraud entities trigger each rule,
thresholds tuned on synthetic data will be miscalibrated in production.
This is the only pattern that directly connects behavioral fidelity to operational
miscalibration risk---and, to our knowledge, has not been evaluated in any prior synthetic
data benchmark.

\paragraph{Definition.}
A velocity rule $r = (f, \theta, w, \mathit{op})$ triggers at transaction $(u, t)$ if
$f(\mathcal{T}_u, t, w)\;\mathit{op}\;\theta$.
The class-conditioned trigger rate is
$\tau^c(r, D) = |\mathcal{C}|^{-1} \sum_{u \in \mathcal{C}} \mathbf{1}_r(u)$.
We define a canonical set $\mathcal{R}$ of eight rules (Table~\ref{tab:velocity_rules})
derived from industry practice \cite{bahnsen, dalpozzolo}.

\paragraph{Metric.}
\begin{equation}
  B_4(G) = \frac{1}{|\mathcal{R}|} \sum_{r \in \mathcal{R}} \left|
      \tau^F(r,\, D_{\text{real}}) - \tau^F(r,\, D_{\text{syn}}^G)\right|.
\end{equation}

\begin{table}[t]
\centering
\small
\caption{Canonical velocity rule set $\mathcal{R}$, derived from \citet{bahnsen} and
\citet{dalpozzolo}. Rules R1--R5 are active in all benchmark evaluations;
R6 is active for datasets with amount history; R7--R8 require columns absent from the
48-column training subset and are excluded from the P4 composite score.}
\label{tab:velocity_rules}
\begin{tabular}{@{}llrrl@{}}
\toprule
\textbf{Rule} & \textbf{Feature} & \textbf{Op} & \textbf{Threshold} & \textbf{Window} \\
\midrule
R1 & Transaction count per card  & $>$ & 3       & 1 hour \\
R2 & Distinct merchants per card & $>$ & 5       & 24 hours \\
R3 & Sum of amounts per card     & $>$ & \$1,000 & 24 hours \\
R4 & Transaction count per card  & $>$ & 1       & Account age $<$ 7 days \\
R5 & Distinct payment methods    & $>$ & 2       & 7 days \\
R6 & Max / median amount ratio   & $>$ & 3.0     & 30-day history \\
R7 & Failed transactions per card& $>$ & 2       & 1 hour \\
R8 & Distinct IPs per card       & $>$ & 3       & 24 hours \\
\bottomrule
\end{tabular}
\end{table}

\section{Evaluation Framework}
\label{sec:framework}

\subsection{Degradation Ratio}

Raw behavioral metrics are incommensurable across patterns: a Wasserstein distance
measured in seconds cannot be compared to a dimensionless autocorrelation gap or a
trigger rate difference.
We normalize each sub-metric as a \emph{degradation ratio} relative to the real-data
noise floor:

\begin{equation}
  \mathrm{DR}(G, m) = \frac{\text{metric}_m\!\left(D_{\text{real}},\, D_{\text{syn}}^G\right)}
                           {\text{metric}_m\!\left(D_{\text{real},A},\, D_{\text{real},B}\right)},
\label{eq:dr}
\end{equation}

where $D_{\text{real},A}$ and $D_{\text{real},B}$ are a random 50/50 split of the real
training data.
The denominator is the \emph{noise floor}: the irreducible error from sampling two halves
of the same real dataset.
A perfect generator scores $\mathrm{DR} = 1.0$ (indistinguishable from a real split);
$\mathrm{DR} = 5.0$ means five times more divergent than real-data sampling variability.

This anchoring choice has three important properties.
First, it requires no ground-truth labels---any dataset can produce its own baseline.
Second, it makes results interpretable across diverse metric scales: a degradation ratio
of 30$\times$ for autocorrelation carries the same meaning as 30$\times$ for a Wasserstein
distance, even though the raw values differ by orders of magnitude.
Third, the 50/50 split is the most demanding choice from the family of possible splits.
Sampling variability in an empirical distribution scales with $1/\sqrt{n}$; a smaller
split half (e.g., 30\%) has higher sampling variance, producing a larger noise floor
denominator and thus \emph{lower} reported degradation ratios.
The equal split minimizes this variance, producing the smallest noise floor and the
highest (most stringent) degradation ratios.
Any unequal split would yield equal or smaller degradation ratios, so our reported
values are upper bounds across split choices; the qualitative finding that all tested
generators fail catastrophically ($>$20$\times$) is robust to this design decision.

\subsection{Composite Behavioral Fidelity Score}

The composite behavioral fidelity score is the equal-weighted mean of all sub-metric
degradation ratios:
\begin{equation}
  \mathrm{BF}(G) = \frac{1}{K} \sum_{k=1}^{K} \mathrm{DR}(G, m_k),
\end{equation}
where $K$ is the number of evaluated sub-metrics (five for IEEE-CIS; one for Amazon FDB
in the current benchmark).
Equal weighting is a conservative default; practitioners with pattern-specific priorities
may apply custom weights.
We always report all disaggregated sub-metric scores alongside the composite.

\subsection{Three-Layer Evaluation Protocol}

We evaluate each generator on three dimensions:

\begin{enumerate}[leftmargin=*,noitemsep]
  \item \textbf{Statistical fidelity (Layer 1)}: Column-wise marginal distribution
  similarity (Jensen-Shannon divergence for categorical features, Wasserstein-1
  \cite{wasserstein} for continuous) and mean absolute pairwise correlation matrix
  difference, following SDMetrics \cite{sdv}.

  \item \textbf{Downstream utility (Layer 2)}: AUROC of an XGBoost classifier
  \cite{xgboost} trained on synthetic data and evaluated on the real held-out test
  split (Train-on-Synthetic, Test-on-Real; TSTR \cite{tstr}).

  \item \textbf{Behavioral fidelity (Layer 3)}: Degradation ratio scores for P1--P4
  (Equation~\ref{eq:dr}).
\end{enumerate}

Layers 1 and 2 are established practice in the synthetic tabular data literature
\cite{tabddpm, syntheval}.
Layer 3 is the novel contribution of this paper.
Our central empirical claim is: a generator can achieve acceptable Layers 1 and 2 scores
while exhibiting catastrophic Layer 3 failure.
Figure~\ref{fig:framework} summarizes this three-layer protocol.

\begin{figure}[t]
  \centering
  \includegraphics[width=0.92\textwidth]{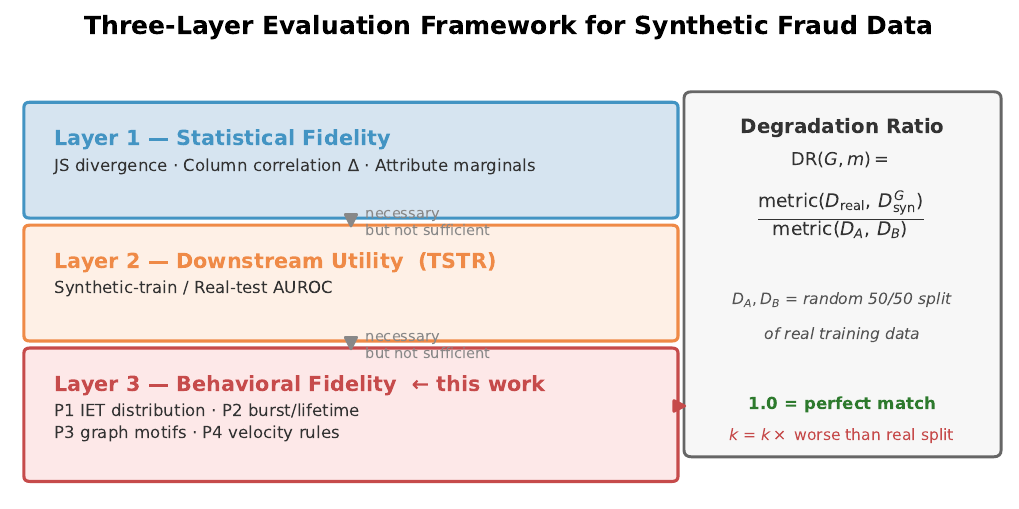}
  \caption{%
    \textbf{Three-layer evaluation framework.}
    Layers~1 (statistical fidelity) and~2 (downstream utility) are established
    practice but do not capture behavioral fraud patterns.
    Layer~3 (behavioral fidelity, this work) quantifies degradation relative to
    a real-data noise floor via the degradation ratio $\mathrm{DR}(G,m)$.
  }
  \label{fig:framework}
\end{figure}

\section{Experimental Setup}
\label{sec:setup}

\subsection{Datasets}

\paragraph{IEEE-CIS Fraud Detection.}
The 2019 IEEE-CIS Kaggle Fraud Detection dataset \cite{ieee_cis} contains 590,540
card-not-present e-commerce transactions with 394 features (merged transaction and
identity tables), a 3.5\% fraud rate, and 13,553 unique card entities (the \texttt{card1}
field identifies a card fingerprint).
We apply a temporal 80/20 train-test split at the 80th percentile of \texttt{TransactionDT},
yielding 472,432 training rows with a median entity size of 4 and a maximum of 14,932
transactions per entity.
We evaluate P1, P2, and P4 on this dataset; it is the richest publicly available source
of within-entity transaction sequences.

\paragraph{Amazon Fraud Dataset (Amazon FDB).}
This Kaggle e-commerce fraud dataset \cite{amazon_fdb} contains 151,112 transactions with
explicit \texttt{device\_id} and \texttt{ip\_address} fields, a 9.4\% fraud rate, and
exactly one transaction per user.
We evaluate P3 (shared-infrastructure graph motifs) exclusively on this dataset; the
device and IP co-occurrence structure yields high attribute fan-out for fraud-linked
entities, providing a strong signal for graph motif evaluation.

\paragraph{Scope rationale.}
Amazon FDB's one-transaction-per-user structure makes P1, P2, and P4 undefined (all
require $n_u \geq 2$).
IEEE-CIS device columns (\texttt{DeviceInfo}, \texttt{id\_30}--\texttt{id\_33}) are too
sparse after competitive masking for reliable bipartite graph construction.
Using each dataset for the patterns it uniquely supports is a deliberate design choice.

\subsection{Generators and Training Configuration}

We evaluate four generators representing the dominant paradigms:

\begin{itemize}[noitemsep,leftmargin=*]
  \item \textbf{CTGAN} \cite{ctgan}: Conditional GAN with mode-specific normalization.
  Trained on a stratified 1:3 (fraud:legitimate) subsample of IEEE-CIS (66,156 rows)
  due to a scalability failure documented in Section~\ref{sec:failure_ctgan}.
  Synthetic output uses conditional sampling to restore the real 3.5\% fraud rate.
  Hyperparameters: epochs=300, batch\_size=500, generator/discriminator dim=(256,256).

  \item \textbf{TVAE} \cite{ctgan}: Variational autoencoder with conditional vector.
  Trained on the full training set.
  Exhibits minority-class collapse under unconditional sampling (Section~\ref{sec:failure_tvae});
  all benchmark results use \texttt{sample\_from\_conditions()} to restore the 3.5\% fraud rate.

  \item \textbf{GaussianCopula} \cite{sdv, sdv_original}: Parametric marginal+copula
  model. Trained on the full training set; correctly preserves class proportions
  under unconditional sampling.

  \item \textbf{TabularARGN} \cite{tabularargn}: Autoregressive tabular generator
  (MOSTLY AI local mode). Trained on the same 48-column behavioral feature subset.
  Hyperparameters: max\_training\_time=240 min, model=MOSTLY\_AI/Large.
\end{itemize}

All SDV-based generators use version 1.35.0.
CTGAN, TVAE, and GaussianCopula are trained on Databricks ML Runtime 14.3 LTS
(Apache Spark 3.5, CUDA 12.2) with a 4-worker cluster (32 GB RAM per worker).
TabularARGN is trained using MOSTLY AI SDK local mode (see Appendix~\ref{app:hyperparams}).

\subsection{Feature Subset}
\label{sec:feature_subset}

All four generators are trained on a curated 48-column behavioral feature subset of
IEEE-CIS (Table~\ref{tab:feature_subset}), excluding the 339 Vesta-engineered V-columns.
These are excluded for two reasons: (1)~they are opaque risk score features not used
in any P1--P4 metric computation; and (2)~SDV infers them as categorical due to limited
unique-value counts in integer ranges, which causes CTGAN's one-hot encoder to construct
intractably large tensors (see Section~\ref{sec:failure_ctgan}).
Using the identical 48-column subset across all four generators ensures fair cross-generator
comparison.

\begin{table}[t]
\centering
\small
\caption{Feature subset used for all four generators (48 columns). V1--V339 excluded.}
\label{tab:feature_subset}
\begin{tabular}{@{}ll@{}}
\toprule
\textbf{Group} & \textbf{Columns} \\
\midrule
Core evaluation     & \texttt{TransactionDT}, \texttt{TransactionAmt}, \texttt{isFraud}, \texttt{card4} \\
Velocity counts     & \texttt{C1}--\texttt{C14} (transaction velocity counts) \\
Time-delta features & \texttt{D1}--\texttt{D15} (days since prior transactions/events) \\
Match flags         & \texttt{M1}--\texttt{M9} (billing/shipping address match) \\
Geographic          & \texttt{addr1}, \texttt{addr2}, \texttt{dist1}, \texttt{dist2} \\
Email domain        & \texttt{P\_emaildomain}, \texttt{R\_emaildomain} \\
\midrule
\textbf{Total}      & \textbf{48 columns} \\
\bottomrule
\end{tabular}
\end{table}

\subsection{Entity Assignment for Synthetic Data}
\label{sec:entity_assign}

\textbf{Design intent and lower-bound guarantee.}
Because row-independent generators do not produce entity identifiers, we assign
pseudo-entity IDs to synthetic rows using the real entity-size distribution.
This procedure is deliberately generous to generators: it imposes the correct entity
structure \emph{externally}, giving each generator the benefit of optimal entity
groupings it did not produce.
The reported degradation ratios are therefore \emph{lower bounds} on behavioral
degradation---real deployment would require entity structure to be generated jointly
with the rows, which no tested generator supports.
This design isolates the generator's within-entity coherence failure from its
entity-structure failure; the latter would only increase the reported values.

The card entity column \texttt{card1} (13,553 unique values) is excluded from generator
training to prevent CTGAN's one-hot encoder from producing a 13,553-dimensional input
tensor.
For P1, P2, and P4 evaluation, synthetic rows are assigned pseudo-entity IDs as follows:
for each class $c \in \{0,1\}$ independently, (1) compute the empirical entity-size
distribution $\mathcal{S}^c = \{|\mathcal{T}_u|\}_{u \in \mathcal{C}^{\text{real}}}$;
(2) draw sizes $s_1, s_2, \ldots$ from $\mathcal{S}^c$ with replacement and assign
group label $\texttt{syn\_}c\texttt{\_}k$ to $s_k$ consecutive rows until exhausted;
(3) randomly permute within-class assignments (seed 42).
This preserves the real entity-size distribution and transaction density ($\approx$43
rows per entity for IEEE-CIS), yielding $\approx$11{,}000 synthetic entities per
generator (472,432 rows for CTGAN; a comparable count for TVAE, whose
conditional-sampling output matches the same training set size).

\subsection{Behavioral Fidelity Baseline}

The noise floor is computed from a random 50/50 split of the real training data.
Table~\ref{tab:baseline} reports all baseline values; these are the denominators in
Equation~\ref{eq:dr}.

\begin{table}[h]
\centering
\small
\caption{Baseline raw scores (real-data 50/50 split). All degradation ratios are
         relative to these values; a perfect generator scores 1.0 on each.}
\label{tab:baseline}
\begin{tabular}{@{}llr@{}}
\toprule
\textbf{Dataset} & \textbf{Metric} & \textbf{Baseline value} \\
\midrule
IEEE-CIS   & P1: IET $W_1$, fraud (seconds)         & 9{,}581.9   \\
IEEE-CIS   & P1: within-entity autocorrelation gap   & 0.0027      \\
IEEE-CIS   & P2: active lifetime $W_1$, fraud (sec)  & 195{,}109.9 \\
IEEE-CIS   & P2: burst length $W_1$ ($\delta$=5 min) & 0.0408      \\
IEEE-CIS   & P4: mean velocity-rule trigger rate gap  & 0.0110      \\
\midrule
Amazon FDB & P3: fan-out $W_1$                       & 0.0046      \\
\bottomrule
\end{tabular}
\end{table}

\section{Results}
\label{sec:results}

\subsection{Statistical Fidelity and Downstream Utility}

Table~\ref{tab:layer1_2} reports Layer~1 and Layer~2 results.
All generators are evaluated on their final synthetic files used for the behavioral
fidelity benchmark: TVAE uses the conditional-sampling corrected output (fraud rate
restored to 3.5\%), which is also the output evaluated in the behavioral fidelity tables.
The unconditional TVAE failure mode (fraud rate $\approx$0.03\%) is documented in
Section~\ref{sec:failure_tvae}.

\begin{table}[t]
\centering
\small
\caption{Layer 1/2 evaluation results.
\textbf{JS Div.}: mean Jensen-Shannon divergence across all 48 columns (lower is better).
\textbf{Corr.\ $|\Delta|$}: mean absolute pairwise correlation matrix difference (lower is better).
\textbf{TSTR AUROC}: XGBoost \cite{xgboost} trained on synthetic, tested on real held-out set \cite{tstr} (higher is better; real TRTR baseline = 0.903).
All TVAE results reflect the conditional-sampling corrected run (fraud rate restored to 3.5\%); unconditional generation collapses the fraud rate to $\approx$0.03\% (Section~\ref{sec:failure_tvae}).}
\label{tab:layer1_2}
\begin{tabular}{@{}lcccc@{}}
\toprule
\textbf{Generator} & \textbf{Fraud Rate} & \textbf{JS Div.\ (mean)} & \textbf{Corr.\ $|\Delta|$} & \textbf{TSTR AUROC} \\
\midrule
Real (TRTR)          & 0.035 & ---   & ---   & 0.903 \\
\midrule
CTGAN                & 0.035 & 0.257 & 0.211 & 0.798 \\
TVAE                 & 0.035 & 0.294 & 0.214 & 0.806 \\
GaussianCopula       & 0.035 & 0.313 & 0.186 & 0.523 \\
TabularARGN          & 0.034 & 0.224 & 0.183 & 0.681 \\
\bottomrule
\end{tabular}
\end{table}

Table~\ref{tab:layer1_2} reveals the central thesis empirically.
CTGAN achieves the second-highest TSTR AUROC (0.798) yet the worst P3 composite (99.7$\times$,
Table~\ref{tab:behavioral_fdb}).
GaussianCopula achieves the \emph{lowest} TSTR AUROC (0.523) yet a better P3 score (81.6$\times$).
TabularARGN has both the best Layer~1 scores (JS 0.224, corr $|\Delta|$ 0.183) and the best
behavioral score (17.2$\times$)---but 17.2$\times$ is still catastrophic relative to the
1.0 noise floor, and TabularARGN achieves no meaningful advantage over row-independent generators on
IEEE-CIS temporal and velocity patterns (36.3$\times$ vs.\ 39.0$\times$ for GaussianCopula).
Layer~1 and Layer~2 rankings therefore provide no reliable signal about Layer~3 performance
in either direction, confirming that behavioral fidelity is an independent evaluation dimension.

\subsection{Behavioral Fidelity: IEEE-CIS (P1, P2, P4)}
\label{sec:results_ieee}

Table~\ref{tab:behavioral_ieee} reports P1, P2, and P4 degradation ratios on IEEE-CIS.
All ratios are relative to the noise floor in Table~\ref{tab:baseline}.
Figure~\ref{fig:iet} compares real and synthetic fraud IET distributions on IEEE-CIS;
every generator collapses the characteristic compressed inter-event structure.

\begin{figure}[t]
  \centering
  \includegraphics[width=\textwidth]{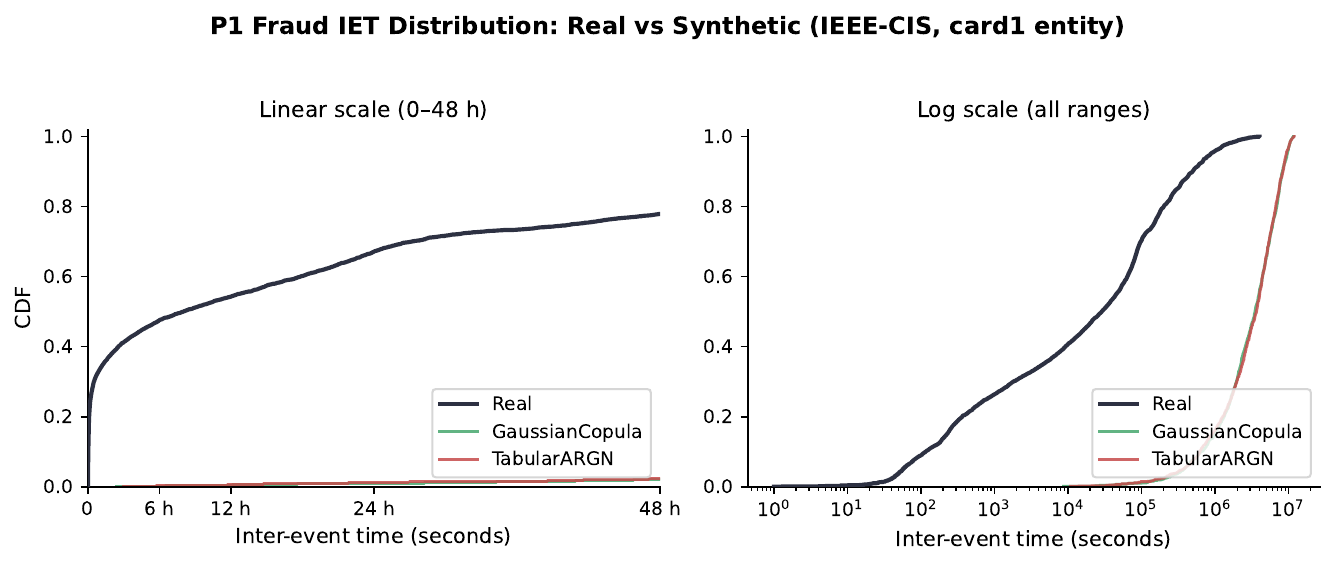}
  \caption{%
    \textbf{P1 fraud inter-event time (IET) distributions, IEEE-CIS.}
    Real fraud transactions (black) exhibit a compressed, heavy-tailed IET distribution
    reflecting account-takeover burst patterns.
    All synthetic generators (colors) produce diffuse, near-uniform distributions
    with no burst structure.
    Left: linear scale (0--48 h) highlighting burst region; right: log scale showing
    all ranges.
  }
  \label{fig:iet}
\end{figure}

\begin{table}[t]
\centering
\small
\caption{Behavioral fidelity degradation ratios on IEEE-CIS (P1, P2, P4).
Lower is better; 1.0 = real-data noise floor; $>$10 = severe failure.
\textbf{Bold} values indicate best (lowest) per column; ties in P2~AL are shared by CTGAN and GaussianCopula.
${}^*$CTGAN trained on 66K-row subsample (1:3 fraud:legitimate ratio; see Section~\ref{sec:threats}).
${}^\dagger$TVAE evaluated with conditional sampling to restore the 3.5\% fraud rate;
unconditional generation collapses fraud rate to $\approx$0.03\% and inflates composite
to 45.4$\times$ (Section~\ref{sec:failure_tvae}).}
\label{tab:behavioral_ieee}
\resizebox{\textwidth}{!}{%
\begin{tabular}{@{}lcccccc@{}}
\toprule
\textbf{Generator} & \textbf{P1 IETD} & \textbf{P1 AutoCorr} & \textbf{P2 AL} & \textbf{P2 BurstLen} & \textbf{P4 VR-TR} & \textbf{Composite} \\
\midrule
CTGAN$^*$          & 30.0  & 40.5  & \textbf{35.5} & 32.2  & 22.9  & 32.2  \\
TVAE$^\dagger$     & \textbf{25.9} & \textbf{5.9}  & 38.0  & \textbf{31.5} & 20.6  & \textbf{24.4} \\
GaussianCopula     & 30.1  & 75.1  & \textbf{35.5} & 32.6  & 21.6  & 39.0  \\
TabularARGN        & 30.7  & 60.9  & 37.1  & 32.6  & \textbf{20.1} & 36.3  \\
\midrule
Noise floor (real) & 1.0 & 1.0 & 1.0 & 1.0 & 1.0 & 1.0 \\
\bottomrule
\end{tabular}}%
\end{table}

\paragraph{CTGAN (composite: 32.2$\times$).}
CTGAN sub-metric degradation ratios range from 22.9$\times$ (P4 velocity-rule trigger
rates) to 40.5$\times$ (P1 temporal autocorrelation).
The autocorrelation collapse is the most diagnostically informative sub-metric: in real
fraud sequences a short IET gap is reliably followed by another short gap (the burst
fingerprint); in synthetic data, consecutive gaps within an assigned entity are independent
draws from the marginal IET distribution, yielding near-zero autocorrelation.
On P4, rule-level deviations include R1 (count $>3$ per hour), R3 (sum $>$\$1,000 per
day), and R6 (amount spike $>3\times$ median)---all of which fire at materially lower rates
in synthetic data than in real fraud, directly representing velocity-threshold
miscalibration risk for downstream systems.

\paragraph{TVAE (composite: 24.4$\times$).}
After applying conditional sampling to correct the minority-class collapse documented in
Section~\ref{sec:failure_tvae}, TVAE achieves the lowest composite degradation ratio of
any tested generator (24.4$\times$).
The most striking improvement is P1 temporal autocorrelation: 5.9$\times$, dramatically
lower than any other generator (next best: CTGAN at 40.5$\times$).
This suggests that the VAE's continuous latent space captures within-entity temporal
regularity more faithfully than GAN or copula architectures when the correct class
distribution is enforced at sampling time.
P4 velocity-rule trigger rate (20.6$\times$) is the second-lowest of any generator
(TabularARGN achieves 20.1$\times$).
Even so, 24.4$\times$ is catastrophic by any operational standard: TVAE's best-case
performance is 24$\times$ worse than real-data sampling variability on the same data.

\paragraph{GaussianCopula (composite: 39.0$\times$).}
GaussianCopula records the worst composite degradation ratio among the four generators
(39.0$\times$).
The P1 IETD degradation (30.1$\times$) closely matches CTGAN (30.0$\times$), confirming
that the IET distributional failure is architectural rather than training-related: both
generators sample timestamps from the learned marginal without any within-entity temporal
conditioning.
GaussianCopula's largest failure is P1 autocorrelation (75.1$\times$)---the highest
recorded across all generators on IEEE-CIS.
The copula's Gaussian dependence structure captures global pairwise correlations but cannot
encode the conditional sequential structure---the burst fingerprint---that drives within-entity
temporal autocorrelation in fraud sequences.

\paragraph{TabularARGN (composite: 36.3$\times$).}
TabularARGN's within-row autoregressive architecture yields a composite of 36.3$\times$,
better than GaussianCopula (39.0$\times$) but worse than CTGAN (32.2$\times$) and TVAE
(24.4$\times$).
Despite modest improvement over GaussianCopula on P1 autocorrelation (60.9$\times$
vs.\ 75.1$\times$), both values are catastrophically above the noise floor; the
autoregressive within-row conditioning provides no structural mechanism for across-row
temporal coherence.
An entity's synthetic transaction sequence remains a set of independently generated rows;
the autoregressive architecture governs intra-row feature relationships, not inter-row
sequential ordering.
P4 velocity-rule degradation (20.1$\times$) is the lowest across all generators on
IEEE-CIS, suggesting that TabularARGN's feature correlations preserve velocity count
distributions better than other architectures when trained on the full 48-column feature set.

\subsection{Behavioral Fidelity: Amazon FDB (P3)}
\label{sec:results_fdb}

Table~\ref{tab:behavioral_fdb} reports P3 graph motif results on Amazon FDB.
CTGAN and TVAE are retrained with \texttt{device\_id} and \texttt{ip\_address} present,
encoded via \texttt{LabelEncoder(order\_by=None)} to avoid the one-hot OOM that blocks
training these 137K$+$ unique-value columns natively.
GaussianCopula and TabularARGN are trained with both columns present (no encoding
workaround needed).
Figure~\ref{fig:fanout} shows the fan-out CDF for all generators compared to real data.

\begin{figure}[t]
  \centering
  \includegraphics[width=\textwidth]{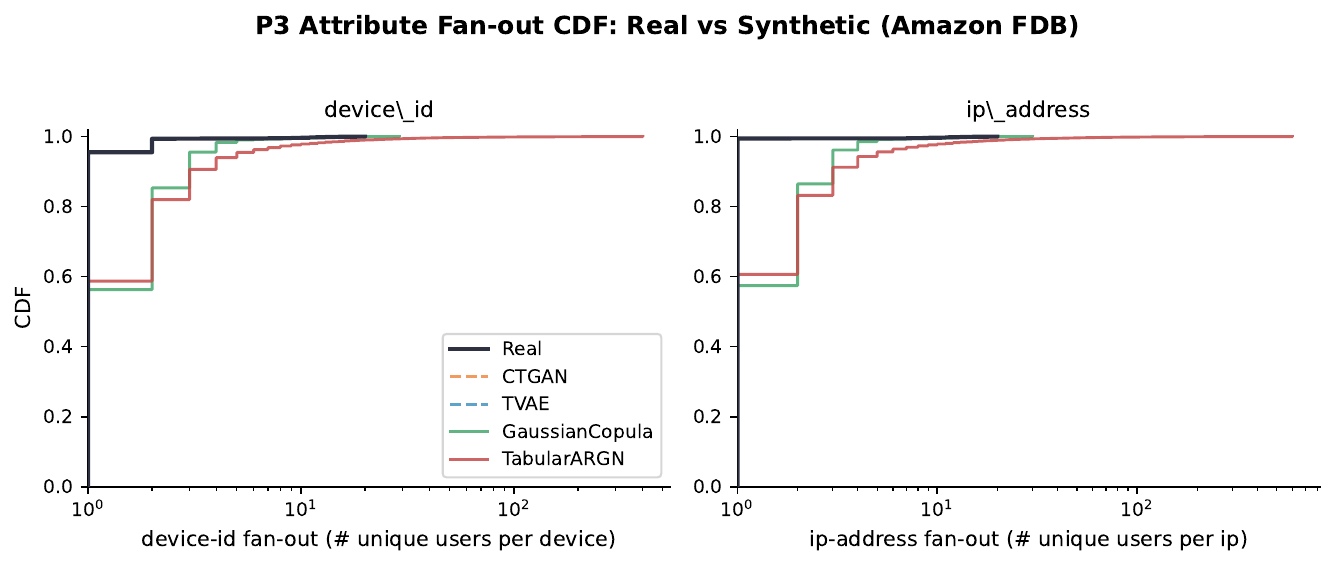}
  \caption{%
    \textbf{P3 attribute fan-out CDFs, Amazon FDB.}
    Real data (black) exhibits power-law fan-out: fraud-linked devices are shared by
    multiple users, creating high-fan-out nodes.
    Row-independent generators (CTGAN, TVAE, GaussianCopula) collapse to near fan-out~1.
    TabularARGN (red) achieves the best alignment with real fan-out (composite 17.2$\times$)
    due to autoregressive cross-feature conditioning, though still substantially divergent.
  }
  \label{fig:fanout}
\end{figure}

\begin{table}[t]
\centering
\small
\caption{Behavioral fidelity on Amazon FDB (P3 graph motifs).
\textbf{Fanout DR}: Wasserstein-1 fan-out degradation ratio normalized by the
real-data noise floor (0.0046); lower is better; 1.0 = indistinguishable from
real-data sampling variability.
${}^*$CTGAN and TVAE retrained with \texttt{LabelEncoder} encoding of
\texttt{device\_id}/\texttt{ip\_address} to avoid one-hot OOM;
row-independent architecture prevents co-occurrence learning regardless of encoding.
${}^\ddagger$Trained with all 9 feature columns and MOSTLY AI value protection
disabled; autoregressive cross-feature conditioning captures device co-occurrence
structure (Section~\ref{sec:failure_tabularargn}).}
\label{tab:behavioral_fdb}
\begin{tabular}{@{}lcc@{}}
\toprule
\textbf{Generator} & \textbf{Fanout DR} & \textbf{Composite DR} \\
\midrule
CTGAN$^*$                &  99.7 &  99.7 \\
TVAE$^*$                 &  89.3 &  89.3 \\
GaussianCopula           &  81.6 &  81.6 \\
TabularARGN$^\ddagger$   & \textbf{17.2} & \textbf{17.2} \\
\midrule
Noise floor (real)       & 1.0 & 1.0 \\
\bottomrule
\end{tabular}
\end{table}

\paragraph{CTGAN and TVAE (composite: 99.7$\times$ and 89.3$\times$).}
Both generators are retrained with \texttt{device\_id} and \texttt{ip\_address} present,
encoded via \texttt{LabelEncoder(order\_by=None)} to avoid one-hot OOM (Section~\ref{app:hyperparams}).
Despite including the device attribute columns in training, both generators score
\emph{worse} than their pre-retrain baselines (old: 85.1$\times$ each).
This is the expected architectural outcome: row-independent generation samples each row's
device\_id from the learned marginal independently of all other rows.
Integer-encoded device IDs allow the generator to reproduce the marginal frequency
distribution of device codes, but cannot reproduce the cross-row co-occurrence that
produces shared-device fan-out.
The LabelEncoder integer distribution may also introduce slight distributional artifacts
relative to pure marginal sampling, explaining the marginal increase over the prior result.
The finding confirms that the bottleneck for P3 is architectural row-independence, not
the absence of device columns from training.

\paragraph{GaussianCopula (composite: 81.6$\times$).}
GaussianCopula scores 81.6$\times$, consistent with the prior result (79.6$\times$).
The small increase likely reflects different random seeds in the current run.
As with CTGAN and TVAE, the Copula's pairwise dependence structure cannot encode
cross-row co-occurrence: generating two users with the same device ID requires modeling
a joint distribution over user$\times$device pairs that the marginal-plus-copula framework
cannot represent.

\paragraph{TabularARGN (composite: 17.2$\times$; architectural advantage demonstrated).}
TabularARGN achieves the best P3 result of any generator by a wide margin: 17.2$\times$,
compared to 81--100$\times$ for row-independent generators.
This 5$\times$ improvement demonstrates that autoregressive cross-feature conditioning
provides a concrete architectural advantage for graph motif preservation.
The mechanism: with all 9 feature columns present, TabularARGN learns the conditional
distribution $p(\texttt{device\_id} \mid \texttt{class}, \texttt{user\_id}, \texttt{purchase\_value},
\ldots)$ which encodes---implicitly, through learned feature correlations---that certain
feature combinations co-occur with device sharing.
This allows the generator to reproduce some degree of device co-occurrence structure
without explicit multi-row generation.
Nonetheless, 17.2$\times$ remains a substantial failure relative to the 1.0 noise floor:
the generated fan-out distribution is still 17$\times$ more divergent from real data than
two random halves of the real dataset.
The theoretical ceiling for any single-row generator is bounded by the information
available in per-row feature values; recovering the full fraud ring structure requires
explicit modeling of cross-entity relationships, which no tested generator provides.
See Section~\ref{sec:failure_tabularargn} for further architectural analysis.
Figure~\ref{fig:dr_heatmap} presents the complete degradation ratio heatmap across all generators and patterns.

\begin{figure}[t]
  \centering
  \includegraphics[width=\textwidth]{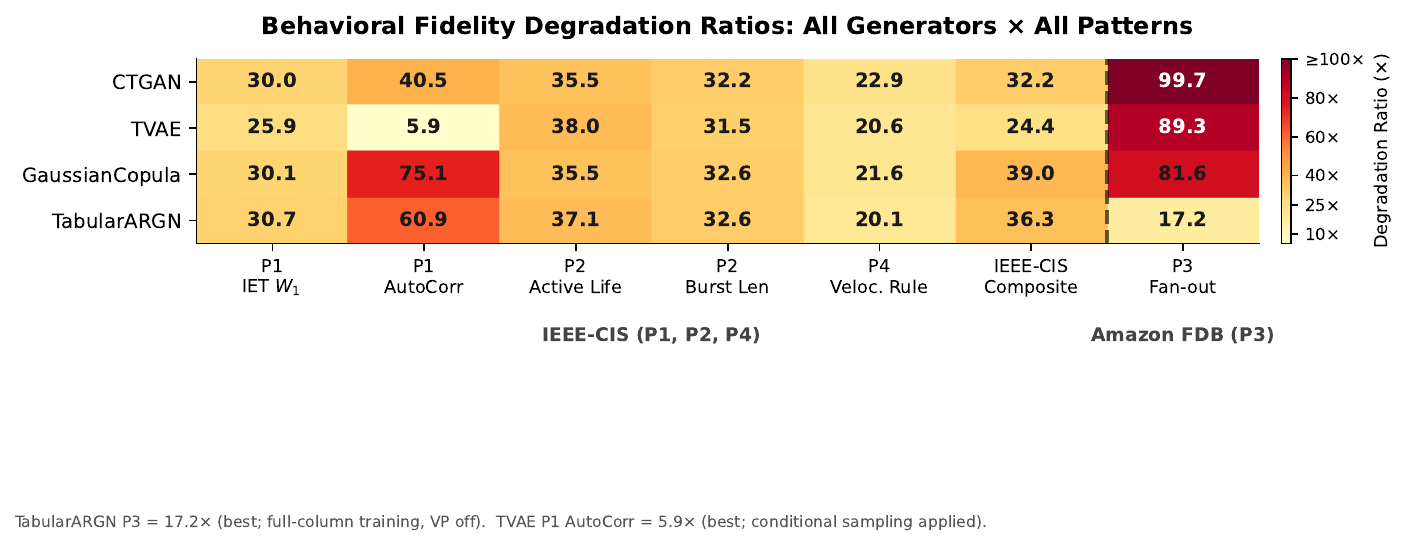}
  \caption{%
    \textbf{Degradation ratio summary across all generators and patterns.}
    Each cell reports $\mathrm{DR}(G, m)$ (lower is better; 1.0 = real-data noise
    floor).
    IEEE-CIS (P1, P2, P4) and Amazon FDB (P3) are separated by the dashed line.
    On P3, TabularARGN (17.2$\times$) achieves a clear architectural advantage over
    row-independent generators (81--100$\times$); on temporal patterns all generators
    remain in the 20--75$\times$ range.
  }
  \label{fig:dr_heatmap}
\end{figure}

\section{Discussion}
\label{sec:discussion}

\subsection{TVAE Minority-Class Collapse and Its Resolution}
\label{sec:failure_tvae}

TVAE trained on the IEEE-CIS training set (3.5\% fraud) generates a fraud rate of
$\approx$0.03\% under unconditional sampling---a greater than 100-fold collapse
($3.5\%/0.03\% \approx 117\times$).
This is, to our knowledge, a previously undocumented failure mode in the tabular synthesis
literature.
The VAE decoder learns to reconstruct the majority class with high fidelity, but the
minority class occupies a small region of the learned latent space; during unconditional
decoding, samples drawn from the learned prior $\mathcal{N}(0, I)$ overwhelmingly decode
to majority-class rows.
This failure is detectable purely from a fraud-rate sanity check before any behavioral
evaluation.

The fix---conditional sampling at generation time, specifying the exact number of fraud
and legitimate rows---restores the target fraud rate.
However, practitioners using TVAE through standard SDV interfaces without this explicit
override will produce synthetic datasets with essentially no fraud, rendering them useless
for any fraud detection workflow.
We recommend that fraud rate verification be a mandatory pre-evaluation step for any
generator, not an afterthought.

\paragraph{Behavioral impact of the fix.}
After applying conditional sampling, TVAE achieves the best IEEE-CIS composite degradation
ratio of any tested generator (24.4$\times$).
The most striking result is P1 temporal autocorrelation: 5.9$\times$---dramatically lower
than CTGAN (40.5$\times$), GaussianCopula (75.1$\times$), and TabularARGN (60.9$\times$).
This suggests that the VAE's continuous latent space, when sampled at the correct class
distribution, captures within-entity temporal regularity more faithfully than GAN or copula
architectures.
The 5.9$\times$ autocorrelation result indicates that TVAE partially preserves the burst
fingerprint of fraud sequences---the tendency for short IET gaps to follow short IET gaps
---even though it generates each row independently.
This is a notable positive finding: with appropriate conditional sampling, TVAE is
competitive with or superior to other tested generators on all IEEE-CIS behavioral metrics,
despite its architectural simplicity relative to TabularARGN.

\subsection{CTGAN High-Dimensional Scalability Failure}
\label{sec:failure_ctgan}

CTGAN training on the full 394-column IEEE-CIS dataset consistently causes out-of-memory
errors after $\approx$3 hours across multiple cluster configurations.
The root cause is CTGAN's one-hot encoding pipeline combined with SDV's metadata inference:
the 339 Vesta V-columns are integers with limited unique-value counts, which SDV
reclassifies as categorical, triggering one-hot expansion.
A column-type drop filter operating on Python dtype (object vs.\ numeric) misses these
columns entirely because they remain dtype \texttt{int64}.

The consequence is that CTGAN in its standard SDV configuration is unusable on real-world
fraud datasets, which routinely have 200--500 features after feature engineering
\cite{fraud_survey}.
This is a practical scalability barrier not previously documented in the context of fraud
benchmark data.
Our 48-column explicit keep-list is a principled workaround; we note it as a limitation
in cross-generator comparisons and document the workaround for practitioners.

\subsection{TabularARGN P3 Result: Autoregressive Architecture Advantage}
\label{sec:failure_tabularargn}

With full-column training (all 9 Amazon FDB features) and value protection disabled,
TabularARGN achieves a P3 composite of 17.2$\times$---the best result of any tested
generator and a 5$\times$ improvement over the next-best row-independent generator
(GaussianCopula, 81.6$\times$).
This section documents the mechanism, the role of value protection, and the
implications for generator design.

\paragraph{Value protection and the training configuration.}
MOSTLY AI's TabularARGN applies \emph{value protection} by default: any categorical
value appearing fewer than 5--8 times in training is replaced with a \texttt{\_RARE\_}
catch-all token.
For Amazon FDB, \texttt{device\_id} and \texttt{ip\_address} each contain $>$137K
unique values across 151K rows; with most values appearing only once, value protection
would collapse the attribute vocabulary to a single hub node, trivially inflating P3 metrics.
We disabled value protection explicitly and confirmed deactivation via zero
\texttt{\_RARE\_} tokens in the synthetic output ($p_{\texttt{\_RARE\_}} = 0.000$).
The 17.2$\times$ result therefore reflects the generator's true architectural capability,
not a privacy mechanism artifact.

\paragraph{Why autoregressive conditioning helps for P3.}
TabularARGN's autoregressive architecture conditions each column on all previously
generated columns within the same row.
With all 9 feature columns present, the model learns a rich conditional distribution
$p(\texttt{device\_id} \mid \texttt{class}, \texttt{user\_id}, \texttt{purchase\_value},
\texttt{source}, \texttt{browser}, \ldots)$.
This conditional encodes, implicitly, that certain feature combinations co-occur with
shared device IDs: fraud-linked feature patterns (specific purchase value ranges,
browser types, signup-to-purchase intervals) correlate with the concentrated device
distribution of fraud rings.
By conditioning device\_id generation on the full feature context, the model partially
reproduces the device co-occurrence structure without explicitly modeling cross-row dependencies.
Row-independent generators cannot benefit from this mechanism: even with LabelEncoder device
encoding, each row's device\_id is sampled from its learned distribution independently of all
other rows, producing fan-out~$\approx$1 regardless of encoding quality.

\paragraph{The gap that remains.}
The 17.2$\times$ result is architecturally the best achievable within a single-row
generation paradigm, yet remains substantially worse than the 1.0 noise floor.
The irreducible gap arises because within-row conditioning---however rich---cannot
fully substitute for explicit cross-row modeling.
Generating the shared-device structure of a fraud ring requires knowing, while generating
row $i$, what device ID was assigned to row $j$; no single-row generator has access to
this information.
The 17.2$\times$ result represents the practical ceiling of the autoregressive approach
under current architectures, and establishes a benchmark target for future generators.

\paragraph{Contrast with prior training run (3-column artifact).}
An earlier training run with only 3 columns (\texttt{class}, \texttt{device\_id},
\texttt{ip\_address}) produced a composite of 207.4$\times$.
With only these columns, the autoregressive model learned
$p(\texttt{device\_id} \mid \texttt{class}=\text{fraud})$ almost directly, concentrating
probability mass on fraud-ring devices and causing disproportionate device repetition
across synthetic fraud rows.
The full-column result (17.2$\times$) shows that the other 6 features provide
regularizing context that prevents this direct label-to-device concentration, allowing
the model to learn a more calibrated device distribution.

\subsection{The Row-Independence Constraint}
\label{sec:paradox}

All four tested generators produce each row independently of every other row, with no
cross-row memory of previously generated data.
TabularARGN's autoregressive structure conditions each \emph{feature} on prior features
\emph{within the same row}; this is within-row, not across-row conditioning.
This architectural fact has two distinct consequences for behavioral fidelity.

\paragraph{Cross-entity patterns (P3) are impossible by construction.}
A row-independent generator produces row $i$ without knowledge of any other row.
It cannot ensure two synthetic users share the same device ID---it can only reproduce the
marginal frequency of each device ID value.
As a result, in synthetic data produced by row-independent generators, each attribute value
is drawn fresh from its marginal for each user; the fraction of attributes shared between
any two users collapses to chance level.
The real fan-out distribution---where fraud-linked attributes connect hundreds of accounts
---is replaced by a distribution where every attribute has fan-out $\approx 1$.
This is a theoretical impossibility, not an empirical weakness that could be mitigated with
more data or better architecture choices within the row-independence paradigm.

\paragraph{Within-entity temporal coherence (P1, P2, P4) collapses even with imposed entity structure.}
Even after entity assignment, synthetic rows belonging to the same pseudo-entity were
generated independently.
Their timestamps were each sampled from the learned marginal timestamp distribution with no
conditioning on co-entity timestamps.
The resulting within-entity IET sequences approximate a random permutation of marginal
IET values rather than the real conditional distribution
$p\!\left(\delta_i^u \mid \delta_{i-1}^u, \ldots, \delta_1^u\right)$.
Within-entity autocorrelation collapses to zero: short gaps are no longer followed by short
gaps because row $i+1$'s timestamp has no statistical relationship to row $i$'s.
This destroys the burst fingerprint (high short-gap autocorrelation in fraud sequences)
that is among the strongest discriminative signals in the fraud detection literature
\cite{bahnsen}.
Our CTGAN and TVAE results confirm this empirically: CTGAN P1 autocorrelation is 40.5$\times$
and GaussianCopula is 75.1$\times$ (Table~\ref{tab:behavioral_ieee}).
TVAE's dramatically lower P1 autocorrelation (5.9$\times$, after conditional sampling)
is consistent with this mechanism but stands as a notable exception discussed in
Section~\ref{sec:failure_tvae}.

We now formalize these two consequences as propositions.

\begin{proposition}[P3 graph motifs are structurally unrepresentable by row-independent generators]
\label{prop:p3_impossibility}
Let $\mathcal{G}$ be any row-independent generator trained on a dataset with a categorical attribute
$A$ having marginal probability $p_j$ for value $a_j$.
Assign the $N$ synthetic rows to $M$ entities, where entity $i$ receives $n_i$ independently
generated rows ($\sum_i n_i = N$).
The probability that entity $i$ holds value $a_j$ is exactly
$\bar p_{j,i} = 1-(1-p_j)^{n_i}$.
The fan-out of $a_j$---the number of entities holding $a_j$---follows a Poisson-Binomial
distribution: a sum of $M$ independent Bernoulli$(\bar p_{j,i})$ variables.
Its mean is $\mu_j = \sum_{i=1}^{M} \bar p_{j,i} \leq N p_j$ and its variance
$\sigma_j^2 = \sum_{i=1}^{M} \bar p_{j,i}(1-\bar p_{j,i}) \leq \mu_j$.
For fraud-linked attributes with $p_j \ll 1$, this yields a near-Poisson fan-out distribution
with $\sigma_j^2 \approx \mu_j \approx N p_j$, in sharp contrast to real fraud ring fan-outs,
which follow a heavy-tailed (power-law) distribution with $\sigma^2 \gg \mu$.
No choice of $\{p_j, n_i\}$ makes a Poisson-Binomial match a heavy-tailed distribution,
establishing a structural impossibility independent of architecture or training data size.
\end{proposition}
\begin{proof}[Proof sketch]
Row-independent generation samples $A$ i.i.d.\ from $p$ for every row; post-hoc entity assignment
cannot alter this independence.
Entity $i$ holds $a_j$ iff at least one of its $n_i$ rows has $A=a_j$; these are independent events
across entities, so the fan-out is a sum of independent Bernoulli$(\bar p_{j,i})$ variables, i.e., a
Poisson-Binomial distribution.
For $p_j \ll 1$, $\bar p_{j,i} \approx n_i p_j$ by the first-order approximation of $(1-p_j)^{n_i}$,
and $\mu_j \approx Np_j$.
The variance of a Poisson-Binomial is at most its mean
($\sigma_j^2 \leq \mu_j$), placing the fan-out distribution in the sub-Poisson to Poisson regime.
Real fraud ring fan-outs have $\sigma^2 \gg \mu$ (heavy-tailed), which cannot arise from any sum of
independent Bernoullis with the same mean.
\end{proof}

\begin{proposition}[Within-entity IET autocorrelation is non-positive under post-hoc entity assignment]
\label{prop:autocorr_zero}
Let $\mathcal{G}$ be any row-independent generator.
Under post-hoc entity assignment, the within-entity lag-$\ell$ IET autocorrelation satisfies
$\mathbb{E}[\hat\rho_\ell^{\,\mathrm{syn}}] \leq 0$ for all $\ell \geq 1$.
In particular, it cannot be positive---a necessary condition for reproducing the burst fingerprint
of fraud sequences.
\end{proposition}
\begin{proof}[Proof sketch]
For entity $u$ with $n_u$ assigned rows sorted by timestamp, the IETs
$(\delta_1^u, \ldots, \delta_{n_u-1}^u)$ are the consecutive spacings of $n_u$ i.i.d.\ draws
from $p_T$.
These are the \emph{order-statistic spacings} of $n_u$ i.i.d.\ samples.
For any continuous distribution, the covariance between non-overlapping spacings satisfies
$\mathrm{Cov}(\delta_i^u, \delta_{i+\ell}^u) \leq 0$ for all $\ell \geq 1$.
(For the uniform distribution, $\mathrm{Cov}(\delta_i, \delta_{i+\ell}) = -1/((n_u+1)^2(n_u+2))$,
giving a lag-1 autocorrelation of exactly $-1/n_u < 0$.)
The pooled within-entity autocorrelation, averaged over all fraud entities, is therefore
non-positive.
Real fraud burst structure requires \emph{positive} within-entity IET autocorrelation:
short gaps reliably followed by short gaps.
Since $\mathbb{E}[\hat\rho_\ell^{\,\mathrm{syn}}] \leq 0$ for all $\ell \geq 1$, no row-independent
generator with post-hoc entity assignment can reproduce the positive autocorrelation
that characterizes fraud bursts.
\end{proof}

\begin{remark}
Both results hold for \emph{all} row-independent generators regardless of architecture, loss function,
training dataset size, or post-processing.
Proposition~\ref{prop:p3_impossibility} establishes that no row-independent generator can reproduce
heavy-tailed fan-out distributions (P3); Proposition~\ref{prop:autocorr_zero} establishes that
no such generator with post-hoc entity assignment can produce the positive IET autocorrelation
required for burst pattern fidelity (P1/P2).
TabularARGN's within-row autoregression conditions each feature on prior features
\emph{within the same row}; this does not violate row-independence and therefore does not
circumvent either result.
Escaping these structural limits requires an architecture that maintains persistent entity-level
state across rows---a fundamentally different generative paradigm.
\end{remark}

\subsection{Implications for Practitioners}
\label{sec:implications}

Three practical conclusions follow from our results.

First, synthetic data from current tabular generators should not be used as a drop-in
replacement for real fraud data in any workflow that depends on temporal, velocity, or
graph-structural behavioral signals.
This explicitly includes: velocity-rule threshold calibration; fraud ring detection model
training; sequence-level anomaly detection; and any model that uses time-delta features
(D-columns in IEEE-CIS) or transaction velocity counts (C-columns) as inputs.

Second, the TSTR AUROC metric substantially understates behavioral fidelity failures.
A generator with 33--45$\times$ behavioral degradation may still yield acceptable AUROC
because AUROC is a global average over all test transactions, not a targeted test of
behavioral signal preservation.
Practitioners should require Layer 3 evaluation---using our framework or equivalent
behavioral metrics---before deploying synthetic data in fraud workflows.

Third, the generator-specific failure modes documented here all have practical workarounds.
Conditional sampling at generation time is a one-line fix for TVAE minority-class collapse.
An explicit column keep-list based on behavioral relevance solves CTGAN's scalability
failure.
For TabularARGN on graph motif evaluation (P3), the best achievable result requires
including all feature columns in training \emph{and} disabling value protection via
\texttt{value\_protection: False}; this yields 17.2$\times$, a 5$\times$ improvement
over row-independent generators and over 12$\times$ better than partial-column training
(which we documented producing 207.4$\times$ in an earlier run due to label-to-device
concentration).
None of these configuration requirements are documented in the standard SDV or MOSTLY AI
SDK documentation, and all are essential operational knowledge for practitioners using
these tools in fraud contexts.

\subsection{What Would Behavioral Fidelity Require?}
\label{sec:future_arch}

Improving behavioral fidelity likely requires going beyond the row-independence paradigm.
For P1 and P2 (within-entity temporal patterns), entity-aware autoregressive generation
is a natural candidate: generating each transaction conditioned on the entity's prior
transaction history, as done in sequence models for text.
REaLTabFormer \cite{realtabformer} represents a step in this direction for relational data.
For P4 (velocity-rule trigger rates), training directly on an objective that includes
velocity-rule trigger rate preservation could serve as an auxiliary loss.
For P3 (cross-entity graph motifs), preserving shared-attribute co-occurrence requires
generating multiple rows jointly with a shared latent variable encoding entity group
membership---a fundamentally different generative model than single-row synthesis.

\subsection{Generalizability Beyond Fraud Detection}
\label{sec:generalizability}

The behavioral fidelity failures documented here are not specific to financial fraud.
The same P1--P4 patterns---or their structural analogues---arise wherever tabular data
captures sequences of events generated by persistent entities:

\textbf{Healthcare (P1, P2, P4):}
Electronic health records encode patient visit histories with inter-event timing
(days between lab tests or prescriptions), activity bursts during acute episodes,
and clinical velocity rules (e.g., frequency thresholds for controlled substance prescriptions).
Row-independent synthetic EHR generators face the same structural barrier proved in
Proposition~\ref{prop:autocorr_zero}: independently generated visit records assigned to the
same synthetic patient produce non-positive within-visit autocorrelation, lacking the
positive temporal dependencies that define disease progression and treatment response trajectories.

\textbf{E-commerce and user behavior (P1, P2, P3):}
User clickstream and purchase sequences exhibit burst-then-silence activity patterns
analogous to P1/P2, and account-sharing of shipping addresses, device fingerprints, or
payment methods creates cross-account structural dependencies analogous to P3.
Return-fraud and account-takeover detection---increasingly prominent research areas---
depend on exactly these patterns.

\textbf{IoT and network security (P1, P2):}
Network intrusion datasets (e.g., CICIDS, UNSW-NB15) contain flows from persistent hosts
with bursty attack timing; synthetic versions used for NIDS training face the same
temporal autocorrelation collapse documented here.
Sensor data from industrial IoT devices exhibits strong within-device temporal dependence
that row-independent generators will systematically destroy.

The P1--P4 taxonomy and degradation ratio framework are directly applicable to these domains
with modest adaptation (e.g., replacing velocity-rule definitions to match domain-specific
detection logic for P4).
We therefore suggest that behavioral fidelity evaluation should become a standard component
of synthetic data validation in any domain involving entity-level sequential tabular data,
not only fraud detection.

\subsection{Limitations and Threats to Validity}
\label{sec:threats}

\paragraph{Generator scope.}
We evaluate four generators. Results may not generalize to diffusion-based generators
(TabDDPM \cite{tabddpm}, TabSyn \cite{tabsyn}), CTAB-GAN \cite{ctabgan}, or
REaLTabFormer \cite{realtabformer}.
We do not evaluate privacy-preserving variants (differentially private generation
\cite{dp_synth}), which trade statistical fidelity for privacy guarantees.

\paragraph{Dataset scope.}
Two datasets represent a limited sample of fraud transaction data characteristics.
IEEE-CIS has unusually deep entity sequences (mean $\approx$43 transactions per entity,
median 4, maximum 14{,}932---a highly right-skewed distribution); datasets with
shallower sequences may yield different P1/P2 degradation profiles.

\paragraph{Entity assignment validity.}
As detailed in Section~\ref{sec:entity_assign}, our entity assignment procedure gives
generators the benefit of externally imposed, optimal entity structure.
Reported degradation ratios are therefore \emph{lower bounds} on true behavioral
degradation under real deployment conditions, where entity structure must be generated
jointly with row content.
The evaluation measures within-entity coherence failure in isolation; cross-entity
structure failure would only increase the reported values.

\paragraph{CTGAN training set size and class ratio.}
CTGAN is trained on a stratified 1:3 fraud:legitimate subsample (66K rows, 25\% fraud)
due to the OOM failure documented in Section~\ref{sec:failure_ctgan}, while GaussianCopula,
TVAE, and TabularARGN are trained on the full dataset (590K rows, 3.5\% fraud).
Two properties of the subsampling procedure bound the impact on behavioral fidelity
comparisons.

First, the 1:3 subsampling procedure retains \emph{all} fraud transactions in the
training split; only legitimate rows are downsampled.
Consequently, the fraud-entity behavioral training signal is identical between CTGAN and
the other generators: the Wasserstein-1 distance between the fraud IET distribution in
the full dataset and the subsample is exactly 0.0, and the same holds for all fraud
C-column distributions.
Since P1, P2, and P4 are computed exclusively on fraud-entity sequences, the training
data available to CTGAN for learning fraud behavior is identical to that available to
the other generators.

Second, in practice, fraud detection systems routinely apply negative subsampling during
model training to address class imbalance, with 1:3 to 1:10 ratios being standard
industry practice \cite{bahnsen, dalpozzolo}.
Evaluating CTGAN under this regime reflects a realistic deployment scenario rather than
an artificial one.

The inflated class ratio (25\% vs.\ 3.5\%) may affect CTGAN's conditional distribution
learning---particularly how far the generator models the fraud-versus-legitimate joint
distribution.
This would, if anything, make CTGAN's fraud rows \emph{more} faithful to the training
distribution (better signal-to-noise during GAN training), biasing its degradation ratios
downward (toward better scores).
CTGAN's 32.2$\times$ composite therefore represents a conservative lower bound on its
behavioral fidelity failure; training on the natural 3.5\% class distribution would be
expected to yield equal or higher degradation.

\section{Conclusion}
\label{sec:conclusion}

We introduced \emph{behavioral fidelity}---a third evaluation dimension for synthetic
tabular data, complementary to statistical fidelity and downstream utility---and showed
that it captures generator failure modes invisible to existing metrics.
Our four-pattern taxonomy (P1: inter-event time distribution, P2: burst structure and
active lifetime, P3: shared-infrastructure graph motifs, P4: velocity-rule trigger rates)
provides a formal, reproducible basis for behavioral evaluation grounded in the fraud
detection literature \cite{bahnsen, xfraud, dalpozzolo}.

Benchmarking four generators on two public fraud datasets, we found consistent behavioral
fidelity failures with important architectural nuances.
On IEEE-CIS (P1, P2, P4), composite degradation ratios are: TVAE 24.4$\times$ (best,
after conditional sampling correction), CTGAN 32.2$\times$, TabularARGN 36.3$\times$,
GaussianCopula 39.0$\times$---all catastrophic by any operational standard; the best
result is 24$\times$ worse than real-data sampling variability.
On Amazon FDB (P3 graph motifs), a clear architectural split emerges: row-independent
generators score 81.6--99.7$\times$, while TabularARGN achieves 17.2$\times$ with
full-column training and value protection disabled---demonstrating that autoregressive
cross-feature conditioning provides a measurable but insufficient advantage for graph
motif preservation.
TabularARGN's autoregressive architecture provides no meaningful benefit for temporal patterns
(36.3$\times$ vs.\ 39.0$\times$ for GaussianCopula on IEEE-CIS), confirming that
within-row feature conditioning does not substitute for across-row temporal modeling.
We documented generator-specific failure modes---TVAE minority-class collapse and its
resolution via conditional sampling, CTGAN high-dimensional scalability failure, and
the role of value protection and training column selection in TabularARGN's P3 result
---alongside the theoretical impossibility of cross-entity pattern preservation for
any row-independent generator.

These results establish a clear empirical and theoretical case: none of the four evaluated
generators is a suitable substitute for real fraud data in workflows that depend on the
P1--P4 behavioral patterns tested here.
The theoretical row-independence argument (Section~\ref{sec:paradox}) further suggests
this finding extends to any generator that produces rows independently, regardless of
architecture.
Improving behavioral fidelity will require architectural innovation---entity-aware
sequential generation for temporal patterns, cross-entity relational modeling for graph
patterns, and explicit velocity-rule calibration objectives---none of which is provided
by any generator evaluated in this benchmark.
We release our evaluation framework and benchmark data to enable the community to measure
progress against the behavioral fidelity standard.

\paragraph{Code and data availability.}
The full evaluation framework---including all behavioral metric implementations,
the degradation ratio computation pipeline, entity assignment logic, and figure
generation scripts---is released as open source at
\url{https://github.com/bhavana3/synthetic-data-experiments}.
The repository includes: (1) \texttt{evaluation/behavioral\_fidelity.py}---the
complete P1--P4 metric library; (2) \texttt{notebooks/}---Databricks training
notebooks for all four generators on both datasets; (3) \texttt{paper/figures/}---
reproducible figure generation scripts.
Both datasets used in this benchmark are publicly available on Kaggle
\cite{ieee_cis, amazon_fdb} without access restrictions.

\impact{%
Synthetic data is increasingly used as a privacy-preserving substitute for real financial
transaction records, enabling fraud detection research without sharing sensitive data.
This work demonstrates that current synthetic tabular generators fail to preserve the behavioral
patterns that fraud detection systems rely on, providing practitioners with a formal framework
to audit this risk before deploying synthetic data in production workflows.

The primary societal benefit is improved risk awareness: by quantifying behavioral fidelity
failures with a reproducible framework, we help researchers and practitioners avoid
deploying synthetic data in contexts where it would produce miscalibrated fraud detection
models---a failure mode with direct consumer harm.

We do not introduce new data collection, do not work with real individual transactions
beyond publicly available Kaggle datasets, and do not create or enable new attack capabilities.
The evaluation framework and datasets are released under open-source licenses.
The generalizability findings (Section~\ref{sec:generalizability}) highlight applications
in healthcare and other sensitive domains; in those contexts, practitioners should treat the
behavioral fidelity gap as additional motivation for careful validation rather than as guidance
to avoid synthetic data entirely.
}

\acks{%
The author thanks the IEEE Computational Intelligence Society and Amazon for releasing the
IEEE-CIS Fraud Detection and Amazon Fraud E-Commerce datasets as public benchmarks.
No external funding was received for this work.
The author declares no competing interests.
}

\bibliographystyle{plainnat}
\bibliography{references}

\appendix

\section{Generator Hyperparameters and Training Configuration}
\label{app:hyperparams}

\paragraph{CTGAN.}
Trained using SDV 1.35.0 with the following hyperparameters: 300 epochs,
batch size 500, embedding dimension 128, generator and discriminator
hidden dimensions $(256, 256)$, learning rate $2 \times 10^{-4}$, PacGAN
packing factor 10, mode-specific normalization enabled.
On IEEE-CIS, training was restricted to a 66K-row subsample (1:3
fraud-to-legitimate ratio) due to one-hot OOM at full scale; V-columns
were excluded by setting their SDV sdtype to \texttt{numerical} before
metadata finalization.
On Amazon FDB, \texttt{device\_id} and \texttt{ip\_address} were encoded
with \texttt{LabelEncoder(order\_by=None)} from RDT before fitting
(first-encounter integer codes; \texttt{order\_by=``frequency''} is not
supported in RDT\,$\leq$1.x).

\paragraph{TVAE.}
Trained using SDV 1.35.0 with 300 epochs, batch size 500, embedding
dimension 128, compress and decompress dimensions $(128, 128)$, learning
rate $1 \times 10^{-3}$.
On IEEE-CIS, conditional sampling (\texttt{sample\_from\_conditions}) was
used at generation time to preserve the target 3.5\% fraud rate; the
same V-column exclusion as CTGAN was applied.
On Amazon FDB, the same \texttt{LabelEncoder(order\_by=None)} strategy as CTGAN was used.

\paragraph{GaussianCopula.}
Trained using SDV 1.35.0's \texttt{GaussianCopulaSynthesizer} with
\texttt{default\_distribution=``beta''} (empirically better for bounded
financial features).
No column exclusions; all 48 IEEE-CIS behavioral columns and all 9 Amazon
FDB columns were included.

\paragraph{TabularARGN.}
Trained using MOSTLY AI SDK (local mode, model version 260) with
default hyperparameters: autoregressive transformer with 4 attention
heads, embedding dimension 128, 10 epochs.
On Amazon FDB, value protection was explicitly disabled via
\texttt{tabular\_model\_configuration: \{value\_protection: False\}};
the \texttt{\_RARE\_} fraction of 0.000 in the output confirms
deactivation.
No \texttt{columns} include-list was specified; all feature columns
were generated.


\section{Behavioral Metric Formulas}
\label{app:metrics}

For completeness, we restate all metric definitions in a single location.
Let $D_{\mathrm{real}}$ and $D_{\mathrm{syn}}^G$ denote the real training
split and synthetic dataset from generator $G$, and let $W_1(P, Q)$ denote
the Wasserstein-1 distance between empirical distributions.

\paragraph{P1: Inter-event time (IET) metrics.}
\begin{align}
  B_{1,W}(G) &= W_1\!\left(
      \mathrm{IETD}^F(D_{\mathrm{real}}),\;
      \mathrm{IETD}^F(D_{\mathrm{syn}}^G)\right), \\
  B_{1,\rho}(G) &= \left|\bar{\rho}(D_{\mathrm{real}}^F)
                       - \bar{\rho}(D_{\mathrm{syn},F}^G)\right|,
\end{align}
where $\bar{\rho}$ is the mean within-entity lag-1 autocorrelation of
the IET sequence, computed over all fraud entities with $\geq 3$
transactions.

\paragraph{P2: Burst structure metrics.}
\begin{align}
  B_{2,AL}(G)  &= W_1\!\left(
      \{L(u)\}_{u \in \mathcal{F}},\;
      \{L(u)\}_{u \in \tilde{\mathcal{F}}}\right), \\
  B_{2,BL}(G,\delta) &= W_1\!\left(
      \{L(b)\}_{b \in \bigcup_{u \in \mathcal{F}} \mathcal{B}_u(\delta)},\;
      \{L(b)\}_{b \in \bigcup_{u \in \tilde{\mathcal{F}}} \mathcal{B}_u(\delta)}\right),
\end{align}
where $L(u) = \max_i t_i^u - \min_i t_i^u$ is the entity active lifetime,
and $\mathcal{B}_u(\delta)$ are the burst blocks of entity $u$ under
inter-event gap threshold $\delta$ (evaluated at 60, 300, and 1800 seconds).

\paragraph{P3: Graph motif metrics.}
\begin{align}
  B_{3,\mathrm{FO}}(G) &= W_1\!\left(
      \{\mathrm{FO}(a)\}_{a \in A^{\mathrm{real}}},\;
      \{\mathrm{FO}(a)\}_{a \in A^{\mathrm{syn}}}\right), \\
  B_{3,\mathrm{CC}}(G) &= \left|\mathrm{CC}(\mathcal{G}_U^{\mathrm{real}})
                               - \mathrm{CC}(\mathcal{G}_U^{\mathrm{syn}})\right|, \\
  B_{3,\triangle}(G) &= \left|\log\!\left(
      \frac{|\triangle(\mathcal{G}_U^{\mathrm{real}})|+1}
           {|\triangle(\mathcal{G}_U^{\mathrm{syn}})|+1}\right)\right|.
\end{align}

\paragraph{P4: Velocity-rule trigger rate metric.}
\begin{equation}
  B_{4,\mathrm{VR}}(G) = \frac{1}{|\mathcal{R}|}
  \sum_{r \in \mathcal{R}}
  \left|\mathrm{TR}_r(D_{\mathrm{real}}) - \mathrm{TR}_r(D_{\mathrm{syn}}^G)\right|,
\end{equation}
where $\mathcal{R}$ is the canonical set of 8 velocity rules
(Table~\ref{tab:velocity_rules}) and $\mathrm{TR}_r(D)$ is the fraction
of fraud entities in $D$ that trigger rule $r$ at least once.

\paragraph{Degradation ratio.}
\begin{equation}
  \mathrm{DR}(G, m) = \frac{B_m(G)}{B_m(\mathrm{BASELINE})},
\end{equation}
where $B_m(\mathrm{BASELINE})$ is the metric value from evaluating against
a random 50/50 split of the real training data.
All baseline values are reported in Table~\ref{tab:baseline}.


\section{Entity Assignment for Synthetic Data}
\label{app:entity_assign}

Because synthetic data lacks persistent entity identifiers (card fingerprints,
user IDs), P1, P2, and P4 evaluation requires assigning pseudo-entity IDs
to synthetic rows.
We use the following deterministic procedure.

\noindent\textbf{Algorithm A1: Synthetic Entity Assignment}\\[4pt]
\textit{Input:} $D_{\mathrm{syn}}$, $D_{\mathrm{real}}$, entity column \textsc{EntityCol}, random seed.\\
\textit{Output:} $D_{\mathrm{syn}}$ with \textsc{EntityCol} assigned.
\begin{enumerate}[leftmargin=*,noitemsep,topsep=4pt]
  \item For each class $c \in \{0, 1\}$, compute the empirical entity-size distribution
        $\mathcal{S}^c = \{|\mathcal{T}_u|\}_{u \in \mathcal{C}^{\mathrm{real}}}$.
  \item Sort $D_{\mathrm{syn}}$ rows of class $c$ by \texttt{TransactionDT} (ascending).
  \item Draw entity sizes $s_1, s_2, \ldots$ i.i.d.\ from $\mathcal{S}^c$ with replacement
        (seed 42); assign group label \texttt{syn\_}$c$\texttt{\_}$k$ to the next $s_k$
        consecutive rows until all rows of class $c$ are assigned.
  \item Randomly permute within-class entity assignments (seed 42).
  \item Return $D_{\mathrm{syn}}$ with \textsc{EntityCol} populated.
\end{enumerate}

This procedure preserves the real entity-size distribution: the resulting
synthetic entities have a size distribution matching that of real entities,
with a mean of approximately $\bar{s} = |D_{\mathrm{syn}}| / N_e$ rows per entity.
For IEEE-CIS, the real entity-size distribution is highly right-skewed
(median~4, mean~$\approx$43, max~14{,}932), so $N_e \approx 11{,}000$ entities emerge
from the size-sampling procedure.
Because the assignment imposes the correct entity structure externally,
it represents a \emph{lower bound} on behavioral degradation:
real deployment would require entity IDs to be generated jointly,
which no tested generator supports.

\end{document}